\documentclass{article}

\usepackage{amsmath, amsthm, amssymb,amsfonts, enumitem, cite, url}
\usepackage{graphicx}
\usepackage{thmtools}

\begin{document}

\makeatletter	   
\renewcommand{\ps@plain}{%
     \renewcommand{\@oddhead}{\textrm{The Generalized Mean Information Coefficient}\hfil\textrm{\thepage}}%
     \renewcommand{\@evenhead}{\@oddhead}%
     \renewcommand{\@oddfoot}{}
     \renewcommand{\@evenfoot}{\@oddfoot}}
\makeatother     
\newcommand{\var}{\mathrm{var}}
\newcommand{\cov}{\mathrm{cov}}
\newcommand{\Id}{\mathrm{Id}}
\newcommand{\E}{\mathbb{E}}
\newcommand{\I}{\mathrm{I}}
\newcommand{\sH}{\mathrm{H}}
\newcommand{\Ind}[1]{\mathbb{I}\{#1\}}
\newcommand{\MIC}{\mathrm{MIC}}
\newcommand{\Approx}{\mathrm{Approx}}
\newcommand{\MCN}{\mathrm{MCN}}
\newcommand{\GMIC}{\mathrm{GMIC}}
\newcommand{\eqd}{\,{\buildrel d \over =}\,}
\newcommand{\plim}{\mathrm{plim}}
\renewcommand \thesection{\Roman{section}.}
\renewcommand \thesubsection{\Alph{subsection}.}
\renewcommand \thesubsubsection{\arabic{subsubsection}.}
\providecommand{\norm}[1]{\lVert#1\rVert}

\title{The Generalized Mean Information Coefficient}
\author{Alexander Luedtke, Linh Tran
\\
Division of Biostatistics, University of California, Berkeley
 \\
 {\tt aluedtke@berkeley.edu} 
 }
 \date{}

\maketitle

\begin{abstract}
\noindent \textbf{Objective:} Reshef \& Reshef recently published a paper in which they present a method called the Maximal Information Coefficient (MIC) that can detect all forms of statistical dependence between pairs of variables as sample size goes to infinity. While this method has been praised by some, it has also been criticized for its lack of power in finite samples. We seek to modify MIC so that it has higher power in detecting associations for limited sample sizes.\\
\textbf{Methods}: Here we present the Generalized Mean Information Coefficient (GMIC), a generalization of MIC which incorporates a tuning parameter that can be used to modify the complexity of the association favored by the measure. We define GMIC and prove it maintains several key asymptotic properties of MIC. Its increased power over MIC is demonstrated using a simulation of eight different functional relationships at sixty different noise levels. The results are compared to the Pearson correlation, distance correlation, and MIC.\\
\textbf{Results:} Simulation results suggest that while generally GMIC has slightly lower power than the distance correlation measure, it achieves higher power than MIC for many forms of underlying association. For some functional relationships, GMIC surpasses all other statistics calculated. Preliminary results suggest choosing a moderate value of the tuning parameter for GMIC will yield a test that is robust across underlying relationships.\\
\textbf{Conclusion:} GMIC is a promising new method that mitigates the power issues suffered by MIC, at the possible expense of equitability. Nonetheless, distance correlation was in our simulations more powerful for many forms of underlying relationships. At a minimum, this work motivates further consideration of maximal information-based nonparametric exploration (MINE) methods as statistical tests of independence.
\end{abstract}

\newtheorem{theorem}{Theorem}
\newtheorem{corollary}{Corollary}
\newtheorem{prop}{Proposition}
\newtheorem{lemma}{Lemma}
\pagestyle{plain}
\section{Introduction.}
Reshef \& Reshef recently published a paper  called ``Detecting Novel Associations in Large Data Sets", in which they present the Maximal Information Coefficient (MIC). This statistic is able to detect many forms of association between pairs of variables, including all functional and a wide range of non-functional relationships\cite{reshef11}. Additionally, MIC is a provably equitable statistic, in the senes that it gives approximately equal scores to different relationships at equal noise levels\cite{reshef11,reshef2013equitability}. Nonetheless, the method has its shortcomings. Shortly after publication, Simon \& Tibshirani published an analysis which showed that, as a test of association, MIC performed worse in terms of power than several other tests of association for many realistic functional forms\cite{simon11}. It is even sometimes less powerful than the Pearson correlation. Simon notes that basic statistical experience suggests testing against such a wide range of alternatives should reduce power in many realistic situations. While MIC's ability to detect the superposition of functions and other complex relationships is a novel and exciting development, most researchers want to detect far simpler relationships in their data. The authors of the original MINE paper responded to the criticism of Simon \& Tibshirani and others in a prepared statement on Andrew Gelman's website by arguing that the main contribution of MIC as a statistic is its equitability, even at the cost of statistical power \cite{reshef12}.

In this paper we focus on statistical power rather than equitability. We present the Generalized Mean Information Coefficient (GMIC), a generalization of MIC which incorporates a tuning parameter that can be used to modify the complexity of the association favored by the measure in a controlled way. Our modification falls into the class of maximal information-based nonparametric exploration (MINE) statistics presented in the original publication. Specifically, we aim to:
\begin{enumerate}[noitemsep]
  \item Introduce GMIC.
  \item Prove that our generalization asymptotically detects the same class of associations as MIC.
  \item Provide intuition as to why our generalization should have higher power in most realistic situations.
  \item Provide simulation evidence supporting our claims.
  \item Compare GMIC to other association statistics using simulation.
\end{enumerate}

\section{Methods.}
\label{sec:methods}
\subsubsection{Maximal Information Coefficient.}

Using mutual information (MI) as a foundation, Reshef \& Reshef propose MIC. Before our discussion of MIC, we review the statistic. Suppose $\{(X_i,Y_i)\}_{i=1}^n$ is a set of i.i.d. random variables drawn from some distribution $P$. We let $X=(X_1,...,X_n)$ and $Y=(Y_1,...,Y_n)$. Let  $I(X,Y|G_{i,j})$ represent the estimated mutual information obtained by placing $(X,Y)$ into the bins defined by $G_{i,j}$, a grid with $i\ge 2$ columns and $j\ge 2$ rows. That is,  $I^*(X,Y)_{i,j} = \max_{G_{i,j}} I(X,Y|G_{i,j})$. See Fig. 1 in the original publication for a graphical interpretation of this procedure. Then for all integers $i,j\ge 2$:
$$C(X,Y)_{i,j} = \frac{I^*(X,Y)_{i,j}}{\log\min\{i,j\}}$$
where $C$ is an infinite matrix and $\log$ represents the base 2 logarithm. We start the row and column indexing at $2$ for convenience. Reshef \& Reshef call $C$ the characteristic matrix. Note that $I(X,Y|G_{i,j})\in [0,\log\min\{i,j\}]$ for all $i,j$, and thus $C_{i,j}\in [0,1]$. Then MIC is given by:
$$\MIC(X,Y) = \max_{ij<B(n)}\{C(X,Y)_{i,j}\}$$
where $B(n)\le O(n^{1-\epsilon})$, $0<\epsilon<1$, is the maximal grid size necessary to control the type I error rate asymptotically. Reshef \& Reshef suggest letting $B(n)=n^{0.6}$. Clearly $0\le \MIC\le 1$. An additional desirable property of MIC is that it only depends on the rank-order of $x$ and $y$. Thus if we compute significance by permutation, we can precompute significance cutoffs for fixed sample sizes and values of $B(n)$.

While exact values of the characteristic matrix can be computed, the space of grids to be searched grows exponentially in $n$. To make running the algorithm feasible on large data sets, Reshef \& Reshef propose a dynamic programming algorithm which provides a close approximation to the characteristic matrix. We refer to the $\MIC$ measure that results from using this matrix $\Approx\MIC$.

In the supplementary materials to the main paper, Reshef \& Reshef prove that $\MIC$ has a number of desirable properties. Because we will use these results to prove that our generalization has similar properties, we restate the results here. For simplicity, here we assume that $B(n)=n^{0.6}$ to avoid restating conditions on $B(n)$. The original publication shows that:
\begin{itemize}
	\item If $X$ and $Y$ are statistically independent, then $\Approx\MIC(X,Y)\rightarrow 0$ in probability as $n\rightarrow\infty$.
	\item If $X$ and $Y$ are not statistically independent, then there exists a constant $\zeta>0$ such that $\MIC(X,Y)\ge\zeta$ almost surely as $n\rightarrow\infty$.
	\item Let $f$ be a nowhere constant function on $[0,1]$ and $X$ be a continuous random variable. Then $\MIC(X,f(X))\rightarrow 1$ as $n\rightarrow\infty$ almost surely.
	\item MIC is equitable in the sense that it scores equally noisy of different types similarly\footnote{For insightful debate of the merits of this final claim, see \cite{reshef2013equitability} and \cite{kinney2013equitability}.}.
\end{itemize}
Note that we are only guaranteed that the approximate algorithm converges to zero in probability for statistically independent random variables and not that $\MIC$ converges to zero in probability.

Thus MIC satisfies many desirable properties for an estimator. In the next section, we briefly discuss one of the MINE statistcs introduced by Reshef \& Reshef to motivate our proposal of a generalization of MIC which prefers simpler association relationships.

\subsubsection{Minimum Cell Number.}
To summarize the complexity of an association, Reshef \& Reshef propose the minimum cell number (MCN):
$$\MCN(X,Y;\delta) = \min_{ij<B(n)}\{\log(ij) : C(X,Y)_{i,j}\ge (1-\delta)MIC(X,Y)\}$$
Where the parameter $0<\delta\le 1$ is chosen based on the desired level of robustness. Larger values of the MCN indicate a complex association, while smaller values indicate a simpler association. As can be seen in Table S1 in the Supplementary Online Materials of the original paper, many realistic associations will tend to have lower MCN scores. In the next section we propose a method which penalizes complex associations in finite samples while maintaining several of the desirable asymptotic properties of MIC.

\subsection{Generalized Mean Information Coefficient.}
\label{sec:gmic}
\subsubsection{GMIC Definition.}
Before introducing the GMIC statistic, we define the maximal characteristic matrix as:
$$C^*(X,Y)_{i,j} = \max_{kl\le ij}\{C(X,Y)_{kl}\}$$
for all $i,j\ge 2$. That is, $C^*(X,Y)_{i,j}$ is the maximal normalized mutual information which can be achieved using grid sizes no larger than $ij$. This processing step is needed for our theoretical results about $\GMIC$ asymptotically detecting all forms of statistical dependence to hold. Note that $C^*$ can be computed efficiently from the characteristic matrix $C$ returned by the current MIC implementation. See Figure \ref{fig:charmats} to see how the maximal characteristic matrix differs from the characteristic matrix in two fabricated examples.

We define $\GMIC_p$ as:
\[
\GMIC_p(X,Y) = \left(\frac{1}{Z}\sum_{ij\le B(n)}\left(C^*(X,Y)_{i,j}\right)^p\right)^{1/p}
\]
where $p$ is a tuning parameter in $[-\infty,\infty]$ and $Z=\#\{(i,j) : ij\le B(n)\}$. We use the convention that $0^{1/p}=0$ for all $p\in[-\infty,\infty]$. We take $\GMIC_{-\infty}$ to denote the minimum and $\GMIC_\infty$ to denote the maximum since these values hold in the limit.

Note that $\GMIC_p$ is a generalized mean. For $p=0$ we take $\GMIC_0$ to denote the geometric mean of $C^*$ because this relation holds in the limit as $p\rightarrow 0$. The generalized mean inequality then guarantees that $\GMIC_p(X,Y)\le \GMIC_q(X,Y)$ for all $-\infty\le p<q\le\infty$ (see 3.2.4 in \cite{abramowitz65}). We denote the value of $\GMIC$ computed on the approximate maximal characteristic resulting from Reshef \& Reshef's algorithm as $\Approx\GMIC$. The fact that $C^*(X,Y)_{i,j}\in[0,1]$ shows that $\GMIC_p\in [0,1]$. In the special case of $p=\infty$, $\GMIC_p(X,Y) = \MIC(X,Y)$.

We also denote $\GMIC_{-\infty}$ as MinIC, standing for the Minimal Information Coefficient. When we invoke the generalized mean inequality in the proofs in the appendix we use the more direct $\GMIC_{-\infty}$ to reference this $\GMIC$ when $p=-\infty$ because it fits naturally into the generalized mean framework. Note that MinIC does not return the minimal value from the characteristic matrix $C$, but rather the minimal value from the maximal characteristic matrix $C^*$, i.e. $C^*_{2,2}$.

Examining the definition of $C^*$ shows that the elements of $C^*$ are nondecreasing as $ij$ increases. The characteristic matrix used by MIC, on the other hand, does not necessarily have this monotonic property (see Figure S7 in the supplementary online materials of\cite{reshef11}).

For $p\in(-\infty,\infty)$, $\GMIC_p$ is similar to the MCN in that it takes into account how quickly (in terms of grid size) the characteristic matrix reaches its maximum. Unlike the MCN, $\GMIC_p$ also takes into account the overall magnitude of the characteristic matrix entries at grid sizes smaller than the grid size $i^* j^*$ which maximizes $I^*(X,Y)_{ij}$. In general, for $p\in(-\infty,\infty)$, $\GMIC_p$ can be viewed as a compromise between the value of the maximal mutual information attainable in a $2\times 2$ grid and the maximal mutual information attainable among all grids of size less than $B(n)$. That is, characteristic matrices with larger values of $C_{2,2}$ or which obtain their maximum at smaller grid sizes will in general achieve higher values of $\GMIC_p$.

\subsubsection{Properties of GMIC.}
The monotonic nature of $C^*$ will allow us to prove several desirable asymptotic properties of GMIC. Specifically, we prove the following results:
\begin{itemize}[noitemsep]
  \item Theorem \ref{thm:gmicindepzero} shows that the approximation algorithm for $\GMIC$ converges in probability to $0$ as $n\rightarrow\infty$ for statistically independent random variables.
  \item Theorem \ref{thm:gmicdepzeta} shows that $\GMIC$ approaches some $\zeta>0$ almost surely for statistically dependent random variables.
  \item Theorem \ref{thm:gmicone} gives a sufficient condition under which that $\GMIC_p$ approaches $1$ for all $p$.
  \item Corollary \ref{gmicmono} shows that $\GMIC_p$ approaches $1$ for all noiseless monotonic relationships. \
\end{itemize}
We first show that the approximation algorithm for $\GMIC$ converges in probability to $0$ as $n\rightarrow\infty$ for statistically independent random variables.
\begin{restatable}{apptheorem}{gmicindepzero}
Suppose $(X,Y)=\{(X_i,Y_i)\}_{i=1}^n$ is a set of continuous i.i.d. $P$ random variables where $P$ has support on $[0,1]^2$ and that $X_1$ is statistically independent of $Y_1$. Then $\Approx\GMIC_p(X,Y)\rightarrow 0$ in probability as $n\rightarrow\infty$ for all $1\le p\le\infty$.
\label{thm:gmicindepzero}
\end{restatable}
We now show that $\GMIC$ is bounded away from zero for statistically dependent random variables.
\begin{restatable}{apptheorem}{gmicdepzeta}
Suppose $(X,Y)=\{(X_i,Y_i)\}_{i=1}^n$ is a set of continuous i.i.d. $P$ random variables where $P$ has support on $[0,1]^2$ and that $X_1$ and $Y_1$ are statistically dependent. Then there exists some constant $\zeta>0$ such that $\GMIC_p(X,Y)\ge\zeta$ almost surely as $n\rightarrow\infty$.
\label{thm:gmicdepzeta}
\end{restatable}
We now develop conditions under which $\GMIC_p$ approaches $1$ for all choices of the tuning parameter $p$. Unsurprisingly, these conditions are stronger than the analogous conditions for $\MIC$.
\begin{restatable}{apptheorem}{gmicone}
Suppose $(X,Y)=\{(X_i,Y_i)\}_{i=1}^n$ is a set of continuous i.i.d. $P$ random variables where $P$ has support on $[0,1]^2$. Let $M_X= \{m : Pr(X_1>m)=1/2\}$ and $M_Y= \{m : Pr(Y_1>m)=1/2\}$. Then $\GMIC_p(X,Y)\rightarrow 1$ in probability for all $p$ if one of the following two conditions holds for some $m_1\in M_X$ and $m_2\in M_Y$:
\begin{enumerate}
 \item $Pr(Y_1< m_2|X_1<m_1)=1$ and $Pr(Y_1> m_2|X_1>m_1)=1$
 \item $Pr(Y_1> m_2|X_1<m_1)=1$ and $Pr(Y_1< m_2|X_1>m_1)=1$
\end{enumerate}
\label{thm:gmicone}
\end{restatable}
Figure \ref{fig:optimalgrid} gives examples of relationships that satisfy the conditions of Theorem \ref{thm:gmicone}. The next corollary shows that asymptotically $\GMIC_p$ approaches $1$ for monotonically increasing or decreasing functional relationships.
\begin{corollary}
\label{gmicmono}
Suppose $(X,Y)=\{(X_i,Y_i)\}_{i=1}^n$ is a set of continuous i.i.d. $P$ random variables where $P$ is continuous with support on $[0,1]^2$. Further suppose that $Y_1=f(X_1)$ almost surely for some noiseless monotonic, nowhere constant function $f$. Then $\GMIC_p(X,Y)\rightarrow 1$ for all $p\in[-\infty,\infty]$.
\end{corollary}
\begin{proof}[Sketch of proof]
Note that $f$ satisfies the conditions in Theorem \ref{thm:gmicone}.
\end{proof}

\subsubsection{Hypothesis Testing with GMIC.}
We now consider hypothesis testing with $\GMIC$. Suppose we wish to test the null distribution that $X$ and $Y$ are statistically independent against the alternative that $X$ and $Y$ are statistically dependent. Then a natural test is to reject the null hypothesis for large values of $\GMIC(X,Y)$. Note that for the same reasons discussed in the MIC paper, GMIC is a rank-order statistic. For continuous random variables, rank ties occur with probability zero. Thus, as in the original MIC paper, we can compute significance by computing $\GMIC_p$ for fixed $p$ on draws from the permutation distribution. Trivially, note that we cannot use the null distribution of $\GMIC_p$ to compute the significance of $\GMIC_q$ for some $q\not=p$. Thus we must compute separate empirical null distributions (though perhaps on the same Monte Carlo draws) for each $\GMIC_p$ of interest.

For a data set in which many variable pairs are of interest, Reshef \& Reshef suggest controlling the false discovery rate \cite{benjamini95}. We do not consider the multiple testing problem here.

\subsubsection{Implementation.}
We implemented $\GMIC$ in C by adding the definition to the cmine library \cite{albanese12}. We then modified the \textbf{minerva} package in R to call the modified C library, rather than the original \cite{minerva12}.

\subsection{Squared Pearson Correlation Coefficient.}
We also compare GMIC with two existing association measures that do not belong to the MINE family of statistics. The first of these measures is given by the square of the Pearson correlation coefficient. Although a squared correlation of zero does not necessarily imply that $X$ and $Y$ are independent, we use $r^2$ to test the hypothesis that $X$ and $Y$ are statistically independent against the alternative that they are not because $r^2$ is commonly used to test this hypothesis in practice. Specifically, we reject for large values of $r^2$. Nonetheless, we recognize that it is in fact inappropriate to use Pearson correlation coefficient to test this hypothesis (e.g. for the quadratic relationship in our simulation, see Section \ref{sec:data}\ref{sec:sim}).

\subsection{Distance Correlation.}
The second existing association measure considered is distance correlation as proposed by Sz{\'e}kely et al. in \cite{szekely07}. Specifically, they consider the parameter:
$$\mathrm{dCor}(X_1,Y_1) = \frac{\mathrm{dCov}(X_1,Y_1)}{\sqrt{\mathrm{dVar}(X_1)\mathrm{dVar}(Y_1)}}$$
Where $\mathrm{dCov}$ and $\mathrm{dVar}$ are defined in Definition 2 of Sz{\'e}kely et al. 2007. While we limit the discussion of distance correlation here, we note that it has the desirable property that $\mathrm{dCor}(X_1,Y_1)=0$ if and only if $X_1$ and $Y_1$ are independent. Conditions under which dCor is $1$ are given in Theorem 3 of the Sz{\'e}kely et al. paper. To estimate $\mathrm{dCor}$, we use the empirical distance correlation given in Definition 5 of the Sz{\'e}kely et al. paper.

We use the empirical distance correlation to test the null hypothesis that $X$ and $Y$ are statistically independent against the alternative that they are not. Specifically we reject for large values of the empirical distance correlation.

For our simulations we use the R implementation of distance correlation in the \textbf{energy} package \cite{rizzo11}.

\section{Data.}
\label{sec:data}
\subsection{Simulation.}
\label{sec:sim}
A simulation of eight different functional relationship was implemented to compare the power of each statistic considered. Specifically, we considered:
\begin{equation} 
	\label{eq:funcs}
	\begin{split}
	&Y_i = X_i + \epsilon_i \\
	&Y_i = 4(X_i-.5)^2 + \epsilon_i \\
	&Y_i = 80(X_i-1/3)^3 - 12(X_i-1/3) + 10\epsilon_i \\
	&Y_i = \sin(4X_i \pi) + 2\epsilon_i \\
	&Y_i = \sin(16X_i \pi) + \epsilon_i \\
	&Y_i = X_i^{1/4} + \epsilon_i \\
	&Y_i = (2W_i-1)\sqrt{1-(2X_i-1)^2} + \epsilon_i/4, \text{where } W_i \sim Bern(0.5) \\
	&Y _i= I(X_i>0.5) + 5\epsilon_i \\
	&\text{where } X_i \sim \mathrm{Uniform}(0,1) \text{ and } \epsilon_i \sim N(0,\mathrm{noise}^2)
	\end{split}
\end{equation}
for $i=1,...,n$. These functions were originally considered by Simon \& Tibshirani to assess the statistical power of MIC \cite{simon11}. Plots of all $8$ relationships are shown in Figure \ref{fig:sample}. One thousand runs were
used to estimate each null and alternative distribution at $2000$ observations and $60$ different noise levels. Like Simon and Tibshirani, we determined proper cutoffs for a level $0.05$ test using the $95^{th}$ percentile of $\GMIC_p$ from the marginal empirical distributions for each $p$. That is, we sampled $\tilde{X}$ i.i.d. from Uniform$(0,1)$, generated $Y$ according to the functions in Equation \ref{eq:funcs} using $\tilde{X}$, and  sampled a separate $X$ i.i.d. Uniform$(0,1)$. We then computed $\GMIC$ on both $(\tilde{X},Y)$ and $(X,Y)$. Power was calculated as the proportion of runs in the alternative distribution $(\tilde{X},Y)$ higher than the calculated cutoff from the null distribution $(X,Y)$. To analyze the performance of the GMIC under different values of the tuning parameter, the simulation was completed under approximately $400$ different values of $p$:
$$p\in\{-100,-99,...,-1\}\cup\{1,2,...,100\}\cup\{-0.99,-0.98,-0.01\}\cup\{0.01,0.02,...,0.99\}\cup\{-\infty,\infty\}$$
In the figures $\GMIC_p$ for $p=-\infty$ is referred to as MinIC and $p=\infty$ is referred to by its original name, MIC. Our consideration of smaller increments as we approached $0$ was due to our initial impression that the maximal power would be obtained at this value.

The limiting value approached by GMIC as $n\rightarrow\infty$ under different values of the tuning parameter may vary for many noiseless
relationships. Futhermore, the limiting value may depend on the form of
the relationship. To analyze non-statistical interpretability in practice, we computed the sample mean of GMIC under three different tuning parameter values, namely $-\infty$, $-1$,
$\infty$, and compared to these values to the sample means of distance correlation and the squared Pearson's correlation coefficient.

\section{Results.}
Figure \ref{fig:power} shows the power across noise levels for the various relationships with a sample size of 2000. Note that $\GMIC_{-1}$ has higher power than MIC for all relationships except the high-frequency sine wave. MinIC has lower power for the high-frequency sine and circle relationships. Notably, the distance correlation statistic
still outperforms both $\GMIC_{-1}$ and MinIC for all but one relationship (high-frequency sine). It is also worth noting that while $\GMIC_{-1}$ beats MinIC in terms
of power for the sine period $1/8$ and circle relationships, the opposite holds true for the linear, $x^{1/4}$, and step function relationships. For the quadratic, cubic, and sine period $1/2$ relationships, MinIC and $\GMIC_{-1}$ have only slightly lower power than distance correlation across the various noise levels.

For comparison with the Simon \& Tibshirani commentary, we also ran the simulation at a sample size of $320$. The results, which appear in Figure \ref{fig:power_pre}, are similar to the results at $n=2000$.

Figure \ref{fig:GMIC_power} shows the power for GMIC at each of the approximately $400$
tuning parameter values, plotted at seven different noise levels. Interestingly, a
sharp increase in the power occurs around $-1$ for every relationship considered
except for sine period $1/8$. This is consistent with the
results seen in Figure \ref{fig:power}. The power tends to stays relatively
constant at non-negative values. Note that \textit{a priori} we believed that $\GMIC_\epsilon$ for some small $\epsilon>0$ would yield the most robust statistic. The choice to include $\GMIC_{-1}$ instead of $\GMIC_\epsilon$ was made based on the results of Figure \ref{fig:GMIC_power}. We acknowledge that such a procedure may take away from the generalizability of the results to other relationships, and that further work is needed to find the best choice of the tuning parameter $p$. Nonetheless, $\GMIC_{0.1}$ attained power close to that of $\GMIC_{-1}$ in many situations as evidenced in Figure \ref{fig:GMIC_power}. Further work is needed to determine how to choose $p$ in a more general setting.

Figure \ref{fig:expectation} shows the sample mean of each statistic for the eight relationships of interest under
the first $10$ noise levels as well as without noise. As expected, values for
all statistics decrease as noise level increases. Whereas MIC should approach $1$ as $n$ grows at a noise level of zero for all relationships considered, we have no such guarantee for $\GMIC_{-1}$. Nonetheless, the sample mean of $\GMIC_{-1}$ is close to the sample mean of MinIC across all relationships except the high-frequency sine and the circle. MinIC is considerably smaller than MIC and $\GMIC_{-1}$ across all relationships except the step function and $X^{1/4}$ relationships. Under the linear association, all values approach $1$ as the noise level approaches $0$.

Similarly, the distance correlation coefficient is noticeably lower in all
the noiseless non-monotonic associations tested. Its sample mean is lower
than (or equal to) the expected value of $\GMIC$ for every functional relationship considered.

\section{Discussion.}

The objective of this study was to develop and test a modification to MIC to 
increase its low statistical power. We have defined an alternative with GMIC and shown analytically that $\GMIC$
maintains many of the same nice properties that MIC possesses.
We have also shown by simulation that there are many realistic settings in which our method performs better than MIC in terms of power for various values of the tuning parameter.

As is always the case, the introduction of a tuning parameter $p$ naturally leads to the question of which choice of $p$ is optimal in terms of some criterion, e.g. statistical power. Our initial belief was that the peak would occur at a value near $p=0$,
thereby approximately returning the geometric mean. The geometric mean was appealing because it (informally) represents the midpoint of $[-\infty,\infty]$, and thus seemed to offer a natural compromise between the extremes of MIC and MinIC. In Figure
\ref{fig:GMIC_power}, we see that except for the high frequency sine wave, our
power appears to be robust at approximately $p=-1$. Within the generalized mean framework, $p=-1$ yields a harmonic mean of the maximal characteristic matrix. A more formal study should be conducted in order to verify that $-1$ performs well at other sample sizes and for other realistic relationships. As stated earlier, for most of the relationships considered the power drops dramatically at non-negative values.

In simulation, $\GMIC_{-1}$ largely outperforms MIC for most of the relationships tested. Nonetheless, distance correlation was clearly superior to all other methods considered in terms of power. Whereas $\GMIC_{-1}$ has somewhat similar performance in the quadratic, cubic, and sine (period $1/2$) associations, the distance correlation still outperforms both GMIC and MinIC for all relationships except for the high frequency sine wave.

Practitioners may find equitability to be a valuable property. According to Reshef \& Reshef, an equitable statistic is one that ``should give similar scores to equally noisy relationships of different types" \cite{reshef11}. Equitability allows the user to easily interpret the strength of the bivariate association. Values close to 1 have very low variance in the assocation and values close to 0 have no association. Reshef \& Reshef showed that MIC satisfies a particular equitability criterion. As of now, we do not have any equitability results for GMIC, nor do we expect that $\GMIC$ satisfies an equitability criterion analogous to the one satisfied by MIC.

Without equitability, interpretability may become an issue when testing across large numbers of pairwise associations. A significance level is less useful if one cannot differentiate between the complexity of the relationship and the underlying noise level implied by a coefficient estimate. $\GMIC$ allows for a natural trade-off between equitability and power, as larger values of the tuning parameter will tend towards MIC. If the desire is to gain power, the practitioner may choose a value for $p$ closer to $-1$. 

While $\GMIC$ generally has more power than the MIC, there are still
further advances needed. We now have a statistic that allows the user to trade equitability
for more statistical power. We leave it to the reader to decide whether the trade is
profitable. At a minimum, this work motiviates further consideration of MINE methods as statistical tests of independence.

\subsection*{Acknowledgements.}
AL and LT would like to thank Sandrine Dudoit for valuable discussions. AL gratefully acknowledges the support of the Department of Defense (DoD) through the National Defense Science \& Engineering Graduate Fellowship (NDSEG) Program.

\bibliographystyle{unsrt}
\bibliography{GMICnotes_bib}{}

\section*{Appendix.}
\gmicindepzero*
\begin{proof}
For $p=\infty$, we note that $\Approx\GMIC_\infty\equiv\Approx\MIC$, and thus extend the results from Theorem 1 of the original paper. Suppose $p\in[-\infty,\infty)$. The generalized mean inequality shows that $\Approx\GMIC_p(X,Y)\le \Approx\GMIC_\infty(X,Y)$. Thus for all $\epsilon>0$, $\{(x,y) : \Approx\GMIC_p(x,y)\ge\epsilon\}\subseteq\{(x,y) : \Approx\GMIC_\infty(x,y)\ge\epsilon\}$. Thus for all $\epsilon>0$:
$$Pr(\Approx\GMIC_p(X,Y)\ge\epsilon)\le Pr(\Approx\MIC_\infty(X,Y)\ge\epsilon)\rightarrow 0\textrm{ as }n\rightarrow\infty$$
because $\Approx\MIC_\infty$ is equivalent to $\Approx\GMIC_\infty$. Thus $\Approx\GMIC_p(X,Y)$ converges in probability to $0$ as $n\rightarrow\infty$.
\end{proof}

\gmicdepzeta*
\begin{proof}
We follow the proof of Proposition 6.8 of the supplementary online materials of \cite{reshef11}. Specifically, we use the fact that $C(X,Y)_{2,2}\ge \zeta$ almost surely for some $\zeta>0$. Note that $C^*_{2,2}\equiv C_{2,2}$, and thus $C^*_{2,2}\ge\zeta$ almost surely. By construction of the maximal characteristic matrix $C^*$, $\GMIC_{-\infty}(X,Y)\equiv C^*_{2,2}$. By the generalized mean inequality, $\GMIC_{-\infty}$ provides a lower bound for $\GMIC_p$ for all $p>-\infty$. The result follows immediately.
\end{proof}

\gmicone*
\begin{proof}
We first show that $\GMIC_p(X,Y)\rightarrow 1$ for all $p\in[-\infty,\infty]$ if condition 1 above holds. Suppose there exist some $m_1\in M_X$ and $m_2\in M_Y$ such that $Pr(Y_1< m_2|X_1<m_1)=1$ and $Pr(Y_1> m_2|X_1>m_1)=1$. Let $V= \Ind{X_1<m_1}$ and $W= \Ind{Y_1<m_2}$. We first show that $\I(V,W)=\log 2$, where $\I(V,W)$ represents the true mutual information on the inderlying distribution $P$. Note that $\I(V,W)=\sH(V)-\sH(V|W)$, where $\sH(V)$ represents the Shannon entropy of $V$ and $\sH(V|W)$ represents the conditional entropy of $V$ given $W$. But condition 1 implies that $\sH(V|W)=0$ because $V$ is known with probabilty $1$ given $W$. Condition 1 also implies that the marginal distribution of $V$ is a Bernoulli distribution with probability of success $1/2$, and thus $\sH(V)=\log 2$. It follows that $\I(V,W)=\log 2$.

Applying Lemma 6.2 and Lemma 6.19 from the Supplementary Online Materials of the original paper shows that $\I(X_1,Y_1|G)$ differs from $\I(V,W)$ by an additive factor of order $o(1)$, where $G$ represents the grid defined by $m_1$ and $m_2$. Thus $\I(X_1,Y_1|G)\rightarrow \log 2$, and it follows that $\GMIC_{-\infty}(X,Y)\rightarrow 1$ in probability. By the generalized mean inequality, $\GMIC_p(X,Y)\rightarrow 1$ in probability for all $p\in[-\infty,\infty]$.

A nearly identical argument shows that $\GMIC_p(X,Y)\rightarrow 1$ for all $p\in[-\infty,\infty]$ if condition 2 holds.
\end{proof}

\newpage
\section{Figures.}
\begin{figure}[ht!]
\begin{center}
\includegraphics[width=5.25in]{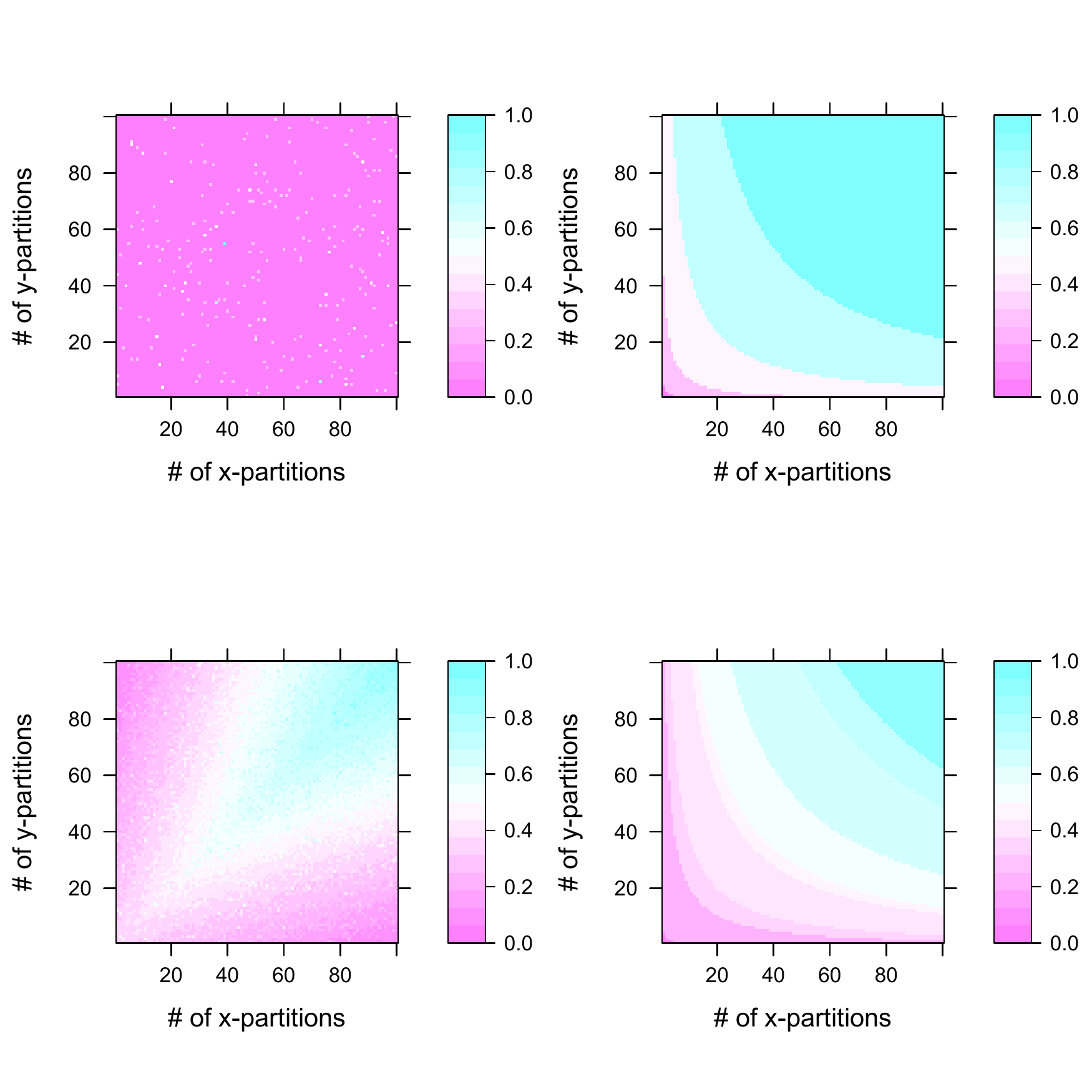}
\end{center}
\caption{On the left are $100\times 100$ subsets of two characteristic matrices. Note that these characteristic matrices do not correspond to actual functional values, rather are just used for display purposes. On the right are the $100\times 100$ subsets of the maximal characteristic matrices corresponding to the adjacent characteristic matrices. Note that the top left characteristic matrix does not appear to have any discernible pattern and only has a few large $I^*$ values, but when processed clearly follows a pattern and is strictly increasing in grid size. The bottom left characteristic appears to have an underlying pattern, but its values are not strictly increasing in grid size. Processing this matrix into the maximal characteristic matrix visually yields a much different result.}
\label{fig:charmats}
\end{figure}

\begin{figure}
\begin{center}
\includegraphics{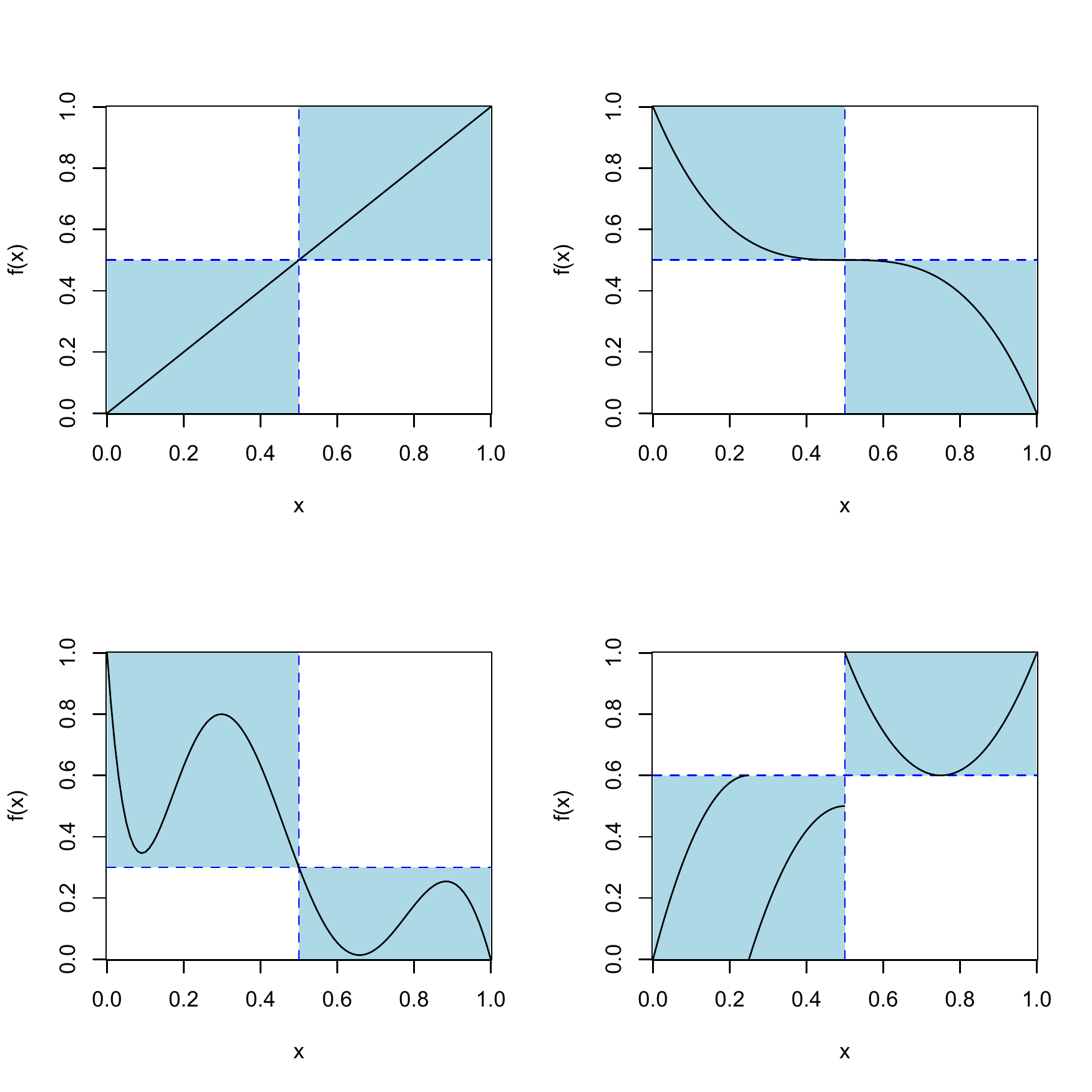}
\end{center}
\caption{Plots of different relationships which satisfy the conditions of Theorem \ref{thm:gmicone} when the median of $X$ falls at $1/2$ (e.g. $X_1\sim\mathrm{Beta}(\beta,\beta)$, $\beta>0$). Note that these functional relationships are both continuous and discontinuous, monotonic and nonmonotonic. The optimal grid of size $4$ is shaded in blue. Given whether or not $Y$ is above or below the dashed horizontal line at $m$, we can say with absolute certainty whether $X$ is to the left or right of the median at $1/2$ in the grids above.}
\label{fig:optimalgrid}
\end{figure}

\begin{figure}
	\begin{center}
		\includegraphics[width=6.5in]{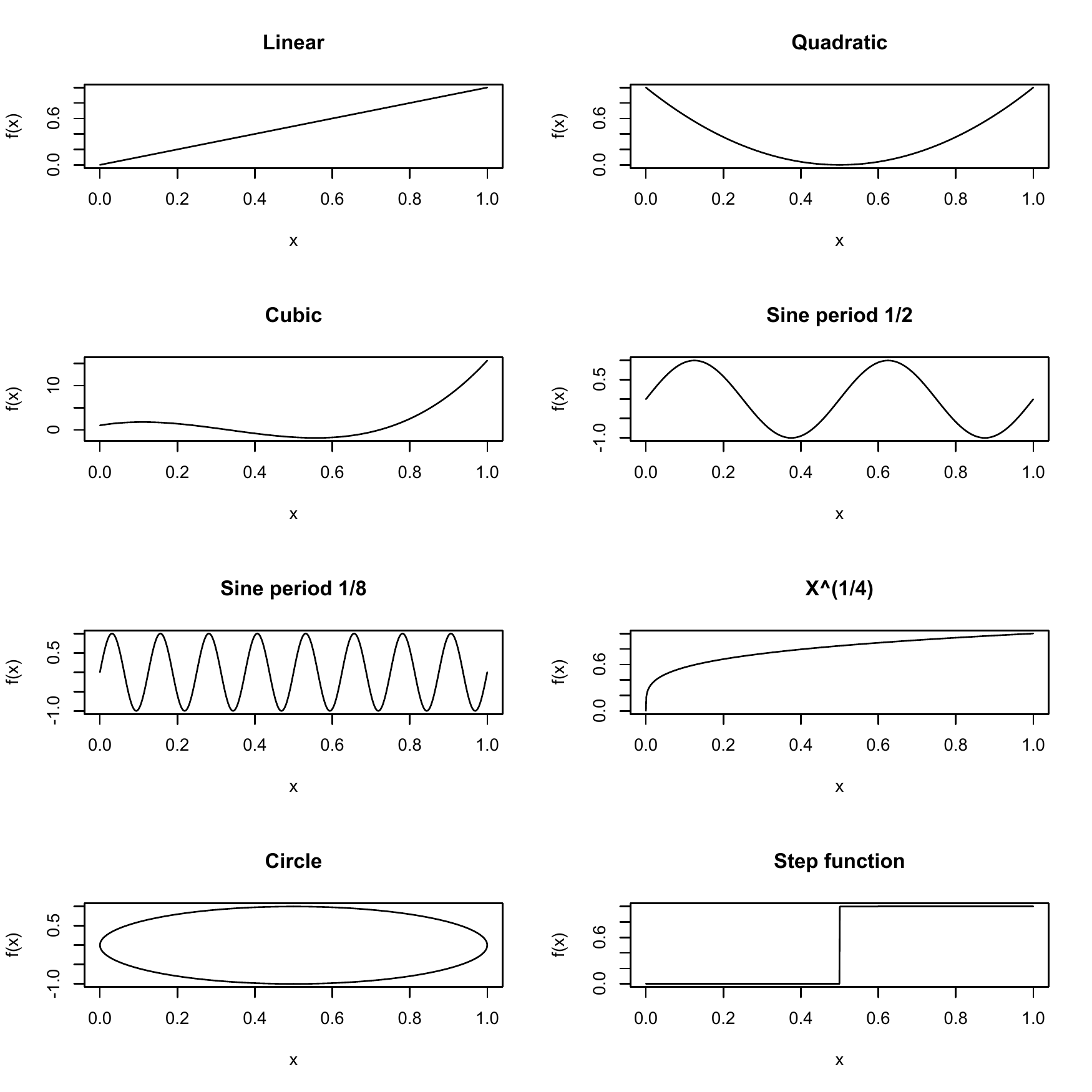}
	\end{center}
	\caption{Sample plots for the eight relationships tested. Note that the relationships are both functional and non-functional, monotonic and non-monotonic.}
	\label{fig:sample}
\end{figure}

\begin{figure}
	\begin{center}
		\includegraphics[width=6.5in]{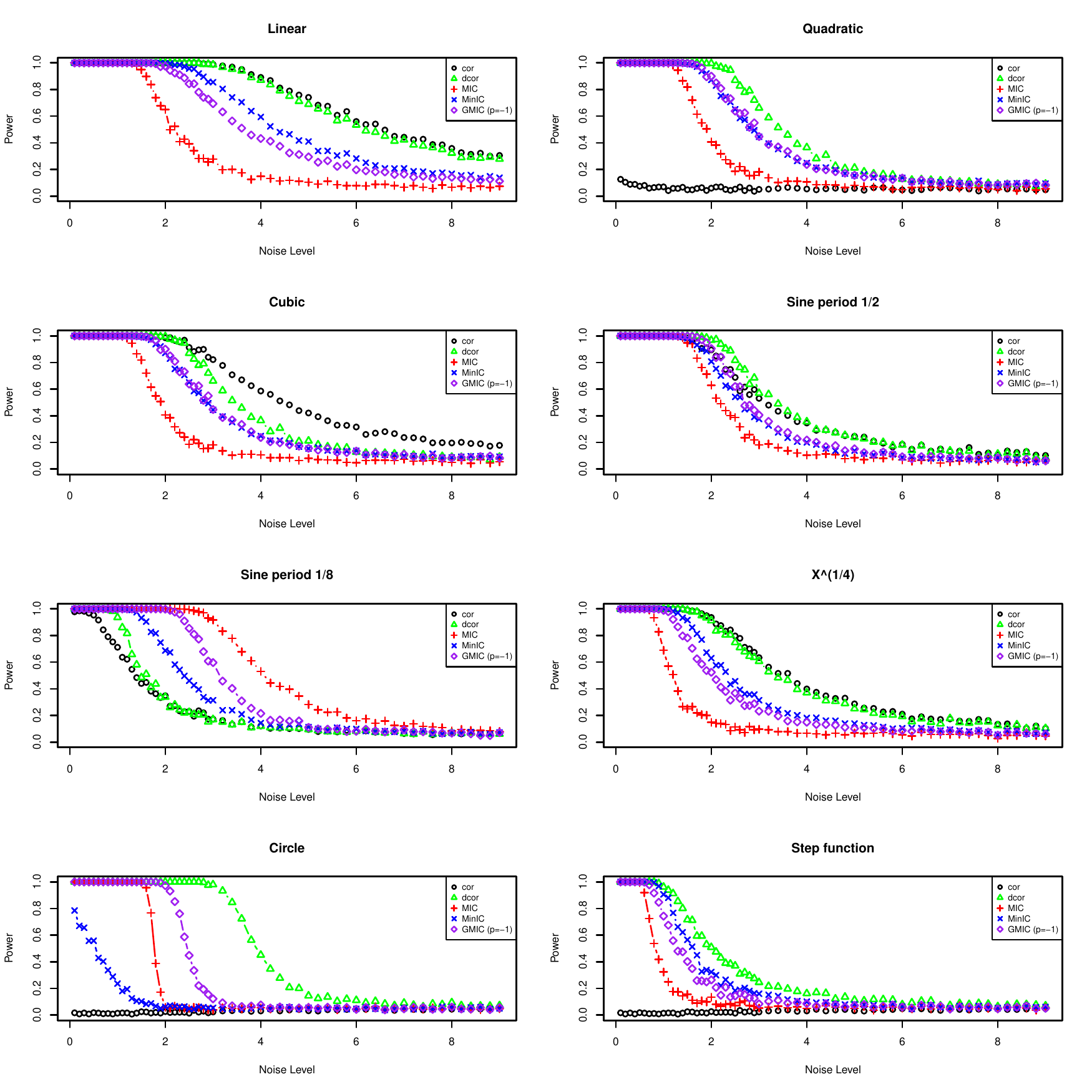}
	\end{center}
	\caption{Simulation results for the eight relationships tested (n=2000). Note that $\GMIC_{-1}$ and MinIC outperform MIC in all cases except for the high-frequency sine wave. Also note that distance correlation outperforms all of the $\GMIC$ family tests except in the case of the high-frequency sine wave.}
	\label{fig:power}
\end{figure}

\begin{figure}
	\begin{center}
		\includegraphics[width=6.5in]{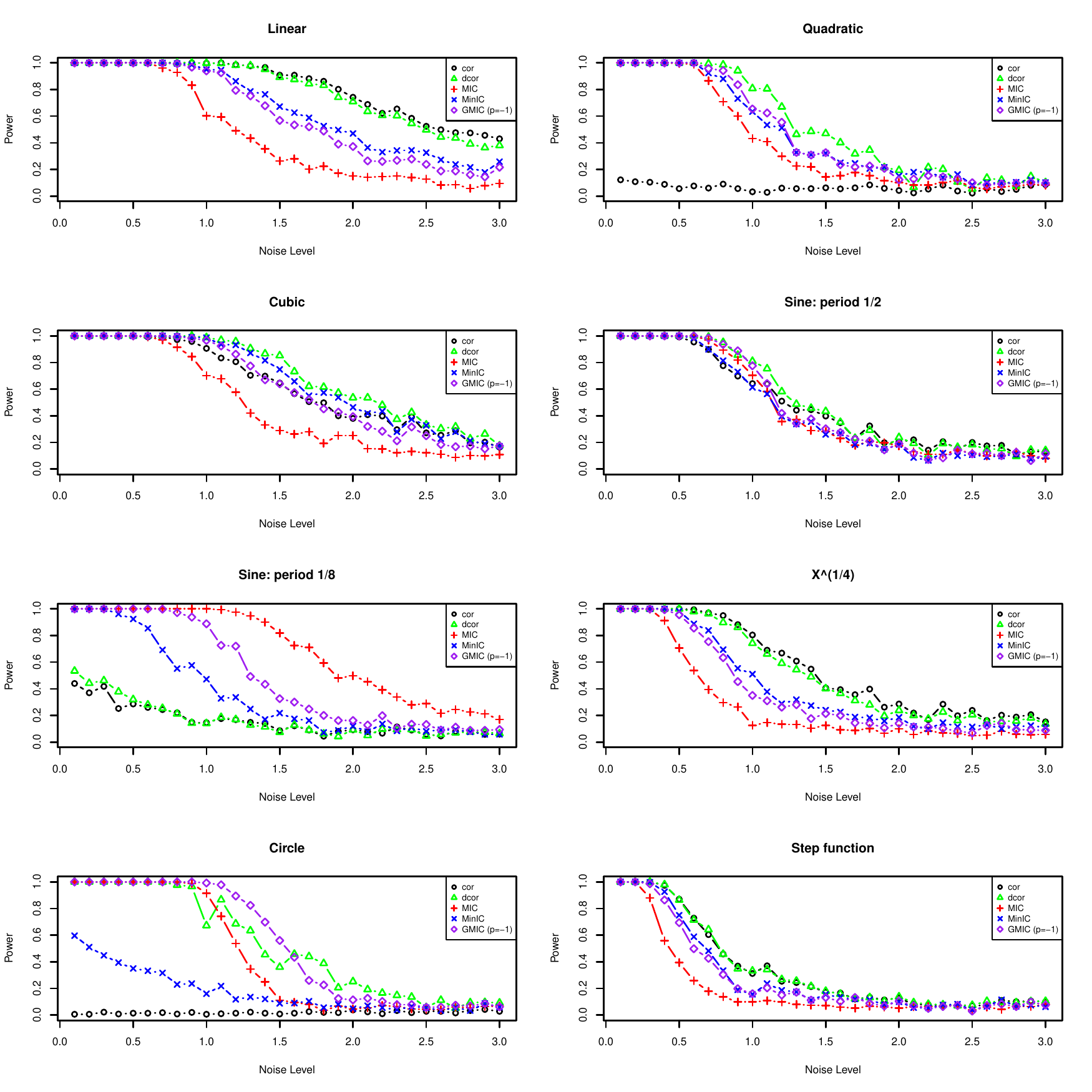}
	\end{center}
	\caption{Simulation results for the eight relationships tested (n=320), for comparison with the Simon \& Tibshirani commentary. Note that the MIC, distance correlation, and Pearson's correlation results match the results in the commentary. Overall the relative power of the methods appears similar to that in Figure \ref{fig:power}.}
	\label{fig:power_pre}
\end{figure}

\begin{figure}
	\begin{center}
		\includegraphics[width=6.5in]{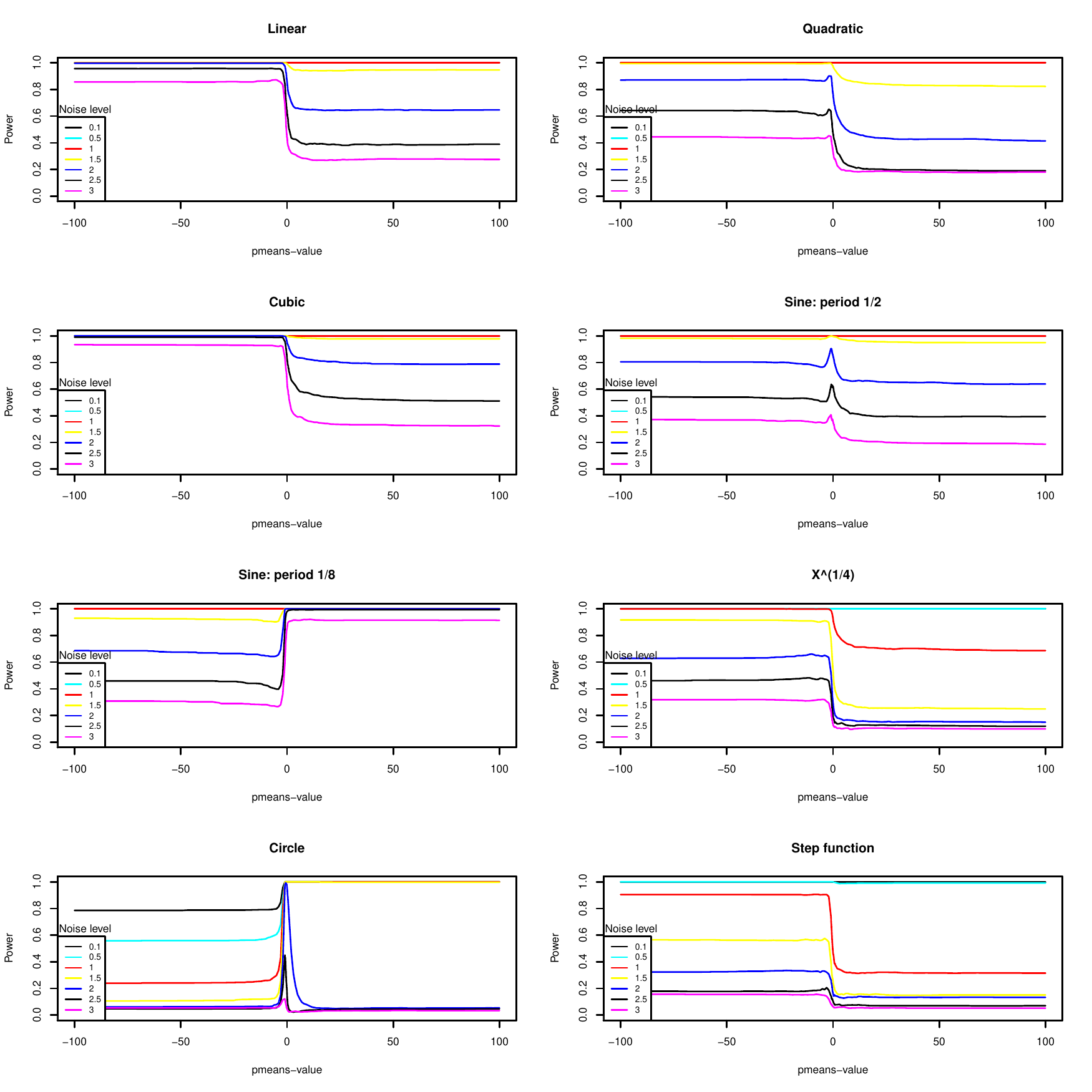}
	\end{center}
	\caption{GMIC power under different tuning parameter values. Note that power seems to be maximal around $p=-1$ for all relationships except the high-frequency sine wave.}
	\label{fig:GMIC_power}
\end{figure}

\begin{figure}
	\begin{center}
		\includegraphics[width=6.5in]{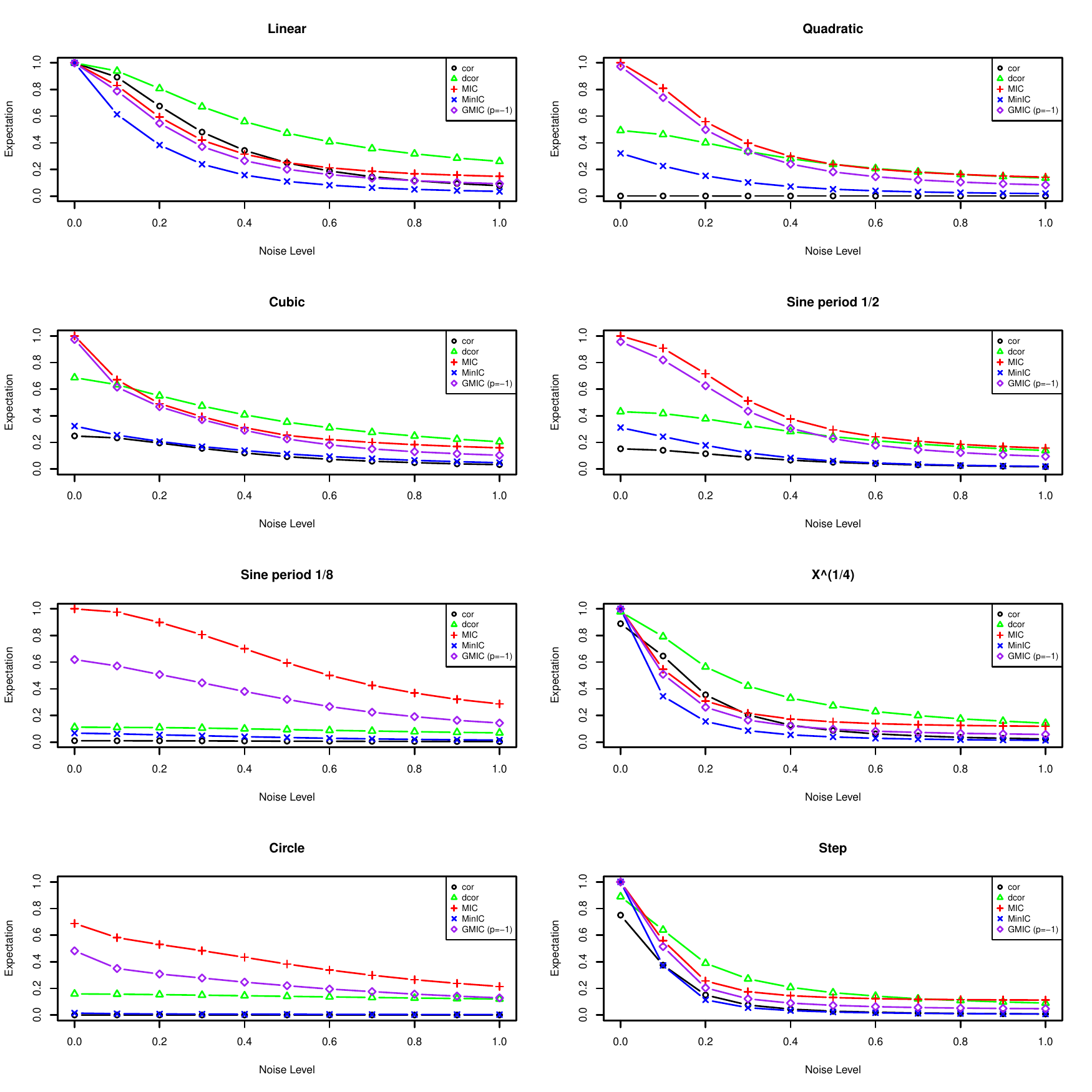}
	\end{center}
	\caption{Sample means of statistic values under different noise levels (n=2000). Note that MIC has the largest value for most settings, a result which is necessarily the case within the class of $\GMIC$ methods. Nonetheless, $\GMIC_{-1}$ is similar to MIC in terms of sample mean for many of the relationships considered.}
	\label{fig:expectation}
\end{figure}

\end{document}